\documentclass{article}

\usepackage[preprint]{neurips_2024}

\usepackage[utf8]{inputenc} %
\usepackage[T1]{fontenc}    %
\usepackage{hyperref}       %
\usepackage{url}            %
\usepackage{booktabs}       %
\usepackage{amsfonts}       %
\usepackage{nicefrac}       %
\usepackage{microtype}      %
\usepackage{xcolor}         %

\usepackage{url}
\usepackage{hyperref}
\usepackage{amsmath}
\usepackage{booktabs}
\usepackage{xspace}
\usepackage{pifont}
\usepackage{threeparttable}
\usepackage{tikz}
\usepackage{graphicx} 
\usepackage{pgfplots}
\usepackage{multirow}
\RequirePackage{algorithm}
\RequirePackage{algorithmic}
\usepackage{enumitem}
\usepackage{amsthm}
\usepackage{wrapfig}
\usepackage{xcolor,colortbl}
\usepackage{amssymb}
\usepackage{subcaption}

\newcommand{\grag}{\textsc{Gnn-Rag}\xspace}
\newcommand{\gragx}{\textsc{Gnn-Rag}}
\newcommand{\lmsr}{$\text{LM}_{\text{SR}}$\xspace}

\usepackage{amsmath,amsfonts,bm}

\def\Figref#1{Figure~\ref{#1}}

\def\Secref#1{Section~\ref{#1}}

\def\eqref#1{equation~\ref{#1}}
\def\Eqref#1{Equation~\ref{#1}}

\def\1{\bm{1}}

\def\vh{{\bm{h}}}

\def\vm{{\bm{m}}}

\def\vq{{\bm{q}}}
\def\vr{{\bm{r}}}

\DeclareMathAlphabet{\mathsfit}{\encodingdefault}{\sfdefault}{m}{sl}
\SetMathAlphabet{\mathsfit}{bold}{\encodingdefault}{\sfdefault}{bx}{n}

\def\gG{{\mathcal{G}}}

\def\gN{{\mathcal{N}}}

\def\gR{{\mathcal{R}}}

\def\gV{{\mathcal{V}}}

\def\sR{{\mathbb{R}}}

\newcommand{\softmax}{\mathrm{softmax}}

\theoremstyle{plain}
\newtheorem{theorem}{Theorem}[section]

\newtheorem{lemma}[theorem]{Lemma}

\theoremstyle{definition}
\newtheorem{definition}[theorem]{Definition}

\theoremstyle{remark}

\title{\grag: Graph Neural  Retrieval for Large Language Model Reasoning}

\author{%
  Costas Mavromatis \\
  University of Minnesota\\
  \texttt{mavro016@umn.edu} \\
  \And
  George Karypis \\
  University of Minnesota\\
  \texttt{karypis@umn.edu} \\
}

\begin{document}

\maketitle

\begin{abstract}
  Knowledge Graphs (KGs) represent human-crafted factual knowledge in the form of triplets \textit{(head, relation, tail)}, which collectively form a graph. Question Answering over KGs (KGQA) is the task  of answering natural questions grounding the reasoning to the information provided by the KG. Large Language Models (LLMs) are the state-of-the-art models for QA tasks due to their remarkable ability to understand natural language. On the other hand, Graph Neural Networks (GNNs) have been widely used for KGQA as they can handle the complex graph information stored in the KG. In this work, we introduce \textsc{Gnn-Rag}, a novel method for combining language understanding abilities of LLMs with the reasoning abilities of GNNs in a retrieval-augmented generation (RAG) style. First, a GNN reasons over a dense KG subgraph to retrieve answer candidates for a given question. Second, the shortest paths in the KG  that connect question entities and answer candidates are extracted to represent KG reasoning paths. The extracted paths are verbalized and given as input for LLM reasoning with RAG. In our \textsc{Gnn-Rag} framework, the GNN acts as a dense subgraph reasoner to extract useful graph information, while the LLM leverages its natural language processing ability for ultimate KGQA. Furthermore, we develop a retrieval augmentation (RA) technique to further boost KGQA performance with \grag. Experimental results show that \textsc{Gnn-Rag} achieves state-of-the-art performance in two widely used KGQA benchmarks (WebQSP and CWQ), outperforming or matching GPT-4 performance with a 7B tuned LLM. In addition, \textsc{Gnn-Rag} excels on multi-hop and multi-entity questions outperforming competing approaches by 8.9--15.5\% points at answer F1. 
We provide the code and KGQA results at \url{https://github.com/cmavro/GNN-RAG}.
  
\end{abstract}

\section{Introduction}

Large Language Models (LLMs)~\citep{brown2020language,bommasani2021opportunities,chowdhery2023palm} are the state-of-the-art models in many NLP tasks due to their remarkable ability to understand natural language. LLM's power stems from pretraining on large corpora of textual data to obtain general human knowledge~\citep{kaplan2020scaling,hoffmann2022training}. However, because pretraining is costly and time-consuming~\citep{gururangan2020don}, LLMs cannot easily adapt to new or in-domain knowledge and are prone to hallucinations~\citep{zhang2023siren}.

Knowledge Graphs (KGs)~\citep{vrandevcic2014wikidata} are databases that store information in structured form that can be easily updated. KGs represent human-crafted factual knowledge in the form of triplets \textit{(head, relation, tail)}, e.g., \texttt{<Jamaica $\rightarrow$ language\_spoken $\rightarrow$ English>}, which collectively form a graph. In the case of KGs, the stored knowledge is updated by fact addition or removal. As KGs capture complex interactions between the stored entities, e.g., multi-hop relations, they are widely used for knowledge-intensive task, such as Question Answering (QA)~\citep{pan2024unifying}.

\begin{wrapfigure}{r}{\wd0} %
    \centering
    \vspace{-0.25in}
    \pgfplotsset{%
    name nodes near coords/.style={
        every node near coord/.append style={
            name=#1-\coordindex,
            alias=#1-last,
        },
    },
    name nodes near coords/.default=coordnode
}

\pgfplotsset{/pgfplots/bar cycle list/.style={/pgfplots/cycle list={
            {olive,fill=olive!40!white,mark=none},
            {brown,fill=brown!40!white,mark=none},
            {teal,fill=teal!50!white,mark=none},
            {black,fill=gray,mark=none},},},}
            
\begin{tikzpicture}
\tikzstyle{every node}=[font=\scriptsize]
    \begin{axis}[
            ybar=.22cm,
            bar width=.3cm,
            enlarge x limits={0.4},
            width=0.4\textwidth,
            height=0.25\textwidth,
            legend style={at={(.45,1.4)},
                anchor=north,legend columns=-1},
            symbolic x coords={ Multi-Hop, Multi-Entity},
            xtick=data,
            nodes near coords,
            nodes near coords align={vertical},
            ymin=30,ymax=90,
            compat=1.5, 
            ylabel={Answer F1 (\%)},
            ylabel near ticks,
        ]
\addplot+[name nodes near coords=bn] coordinates {(Multi-Hop,41.1)[{41.1}] (Multi-Entity,46.9)[{46.9}]};
 \addplot+[name nodes near coords=rn] coordinates {(Multi-Hop, 61.3)[{61.3}] (Multi-Entity,61.7)[{61.7}]};
 \addplot+[name nodes near coords=ln] coordinates {(Multi-Hop,70.2)[{70.2}] (Multi-Entity,77.2)[{77.2}]};

        \legend{No RAG,+KG-RAG,+\textbf{\grag}}
    \end{axis}

\end{tikzpicture}
    \caption{Retrieval effect on  \underline{multi-hop/entity} KGQA. Our \textbf{\grag} \textbf{outperforms} existing KG-RAG methods by 8.9--15.5\% points. }
    \vspace{-0.15in}
    \label{fig:intro}
\end{wrapfigure}
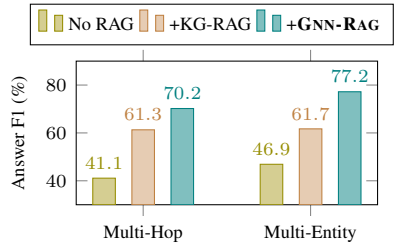

Retrieval-augmented generation (RAG) is a framework that alleviates LLM hallucinations by enriching the input context with up-to-date and accurate information~\citep{lewis2020retrieval}, e.g., obtained from the KG. In the KGQA task, the goal is to answer natural questions grounding the reasoning to the information provided by the KG. For instance, the input for RAG becomes ``\texttt{Knowledge: Jamaica $\rightarrow$ language\_spoken $\rightarrow$ English \textbackslash n Question: Which language do Jamaican people speak?}'', where the LLM has access to KG information for answering the question.

RAG's performance highly depends on the KG facts that are retrieved~\citep{wu2023retrieve}. The challenge is that KGs store complex graph information (they usually consist of millions of facts) and retrieving the right information requires effective graph processing, while retrieving irrelevant information may confuse the LLM during its KGQA reasoning~\citep{he2024gretriever}. Existing retrieval methods that rely on LLMs to retrieve relevant KG information (LLM-based retrieval) underperform on multi-hop KGQA as they cannot handle complex graph information~\citep{baek2023knowledge,luo2024rog} or they need the internal knowledge of very large LMs, e.g., GPT-4, to compensate for missing information during KG retrieval~\citep{sun2024tog}.

In this work, we introduce \grag, a novel method for improving RAG for KGQA. \grag relies on Graph Neural Networks (GNNs)~\citep{mavromatis2022rearev}, which  are powerful graph representation learners, to handle the complex graph information stored in the KG. Although GNNs cannot understand natural language the same way LLMs do, \grag repurposes their graph processing power for retrieval.  First, a GNN reasons over a dense KG subgraph to retrieve answer candidates for a given question. Second, the shortest paths in the KG  that connect question entities and GNN-based answers are extracted to represent useful KG reasoning paths. The extracted paths are verbalized and given as input for LLM reasoning with RAG. 
Furthermore, we show that \grag can be augmented with LLM-based retrievers to further boost KGQA performance.
Experimental results show \grag's superiority over competing RAG-based systems for KGQA by outperforming them by up to 15.5\% points at complex KGQA performance (\Figref{fig:intro}). 
Our \textbf{contributions} are summarized below:
\begin{itemize}
    \item \textbf{Framework}: \grag repurposes GNNs for KGQA retrieval to enhance the reasoning abilities of LLMs. In our \textsc{Gnn-Rag} framework, the GNN acts as a dense subgraph reasoner to extract useful graph information, while the LLM leverages its natural language processing ability for ultimate KGQA.  Moreover, our retrieval analysis (\Secref{sec:analysis}) guides the design of a retrieval augmentation (RA) technique to boost \gragx's performance (\Secref{sec:ret_aug}). 
    \item \textbf{Effectiveness \& Faithfulness}: \grag achieves state-of-the-art performance in two widely
used KGQA benchmarks (WebQSP and CWQ). \grag retrieves multi-hop information that is necessary for faithful LLM reasoning on complex questions (8.9--15.5\% improvement; see \Figref{fig:intro}).
    \item \textbf{Efficiency}: \grag improves vanilla LLMs on KGQA performance without incurring additional LLM calls as existing RAG systems for KGQA require. In addition, \grag outperforms or matches GPT-4 performance with a 7B tuned LLM.
\end{itemize}

\section{Related Work}

\textbf{KGQA Methods}. KGQA methods fall into two categories~\citep{lan2022complex}: (A) Semantic Parsing (SP) methods and (B) Information Retrieval (IR) methods. SP methods~\citep{sun2020sparqa,lan-jiang-2020-query,ye-etal-2022-rng} learn to transform the given question into a query of logical form, e.g., SPARQL query. The transformed query is then executed over the KG to obtain the answers. However, SP methods require ground-truth logical queries for training, which are time-consuming to annotate in practice, and may lead non-executable queries due to syntactical or semantic errors~\citep{das2021case,yu2022decaf}. IR methods~\citep{sun-etal-2018-open,sun2019pullnet} focus on the weakly-supervised KGQA setting, where only question-answer pairs are given for training. IR methods retrieve  KG information, e.g., a KG subgraph~\citep{zhang2022subgraph}, which is used as input during KGQA reasoning. In Appendix~\ref{app:analysis}, we analyze the reasoning abilities of the prevailing models (GNNs \& LLMs) for KGQA, and in  \Secref{sec:gnn-rag}, we propose \grag which leverages the strengths of both of these models.

\textbf{Graph-augmented LMs}. Combining LMs with graphs that store information in natural language is an emerging research area~\citep{jin2023large}. There are two main directions, (i) methods that enhance LMs with \emph{latent} graph information~\citep{zhang2022greaselm,tian2024graph,huang2024can}, e.g., obtained by GNNs, and (ii) methods that insert \emph{verbalized} graph information at the input~\citep{UnifiedSKG,jiang2023structgpt,jin2024graph-cot}, similar to RAG. The methods of the first direction are limited because of the modality mismatch between language and graph, which can lead to inferior performance for knowledge-intensive tasks~\citep{mavromatis2024sempool}. On the other hand, methods of the second direction may fetch noisy information when the underlying graph is large and such information can decrease the LM's reasoning ability~\citep{wu2023retrieve,he2024gretriever}. \grag employs GNNs for information retrieval and RAG for KGQA reasoning, achieving superior performance over existing approaches.

\begin{figure*}[tb]
    \centering
    \includegraphics[width=1\linewidth]{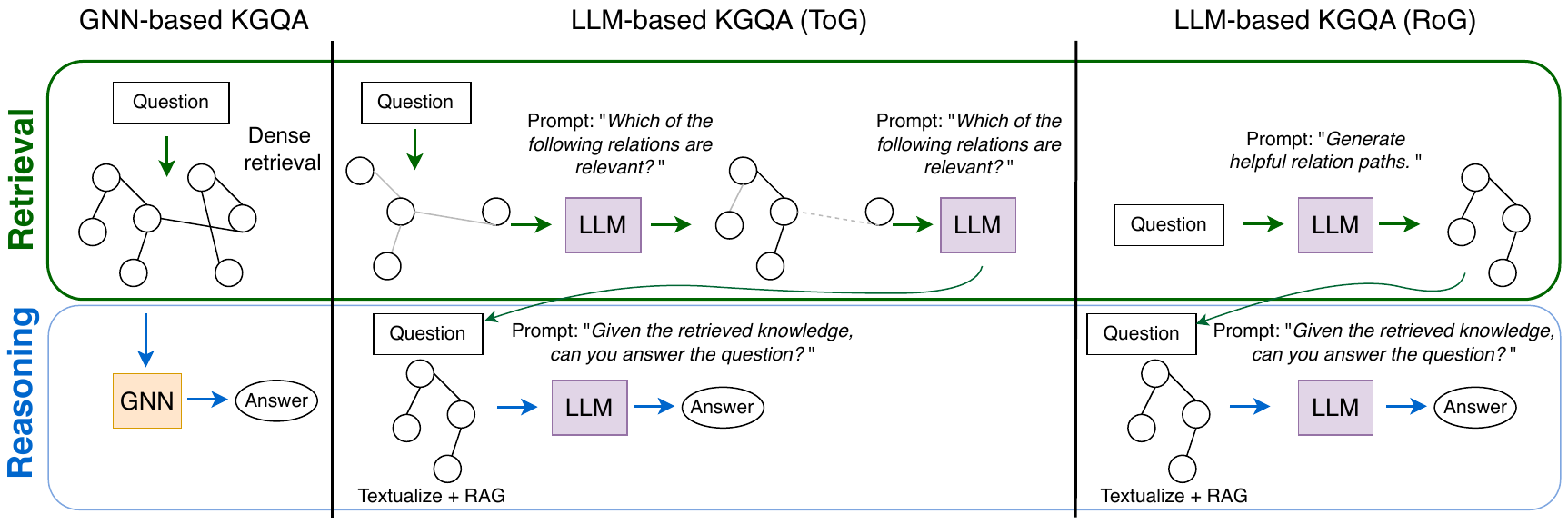}
    \caption{The landscape of existing KGQA methods. GNN-based methods reason on dense subgraphs as they can handle complex and multi-hop graph information. LLM-based methods employ the same LLM for both retrieval and reasoning due to its ability to understand natural language.}
    \label{fig:retrieval-reasoning}
\end{figure*}

\section{Problem Statement \& Background} \label{sec:problem}

\textbf{KGQA}. We are given a KG $\gG$ that contains facts represented as $(v,r,v')$, where $v$ denotes the head entity, $v'$ denotes the tail entity, and $r$ is the corresponding relation between the two entities.
Given $\gG$ and a natural language question $q$, the task of KGQA is to extract a set of entities $\{a\} \in \gG$ that correctly answer $q$. Following previous works~\citep{lan2022complex}, question-answer pairs are given for training, but not the ground-truth paths that lead to the answers. 

\textbf{Retrieval \& Reasoning}. As KGs usually contain millions of facts and nodes, a smaller question-specific subgraph $\gG_q$ is retrieved for a question $q$, e.g., via entity linking and neighbor extraction~\citep{yih2015semantic}. Ideally, all correct answers for the question are contained in the retrieved subgraph,  $\{a\} \in \gG_q$.
The retrieved subgraph $\gG_q$ along with the question $q$ are used as input to a reasoning model, which outputs the correct answer(s). The prevailing reasoning models for the KGQA setting studied are GNNs and LLMs. 

\textbf{GNNs}. KGQA can be regarded as a node classification problem, where KG entities are classified as answers vs. non-answers for a given question. 
GNNs~\cite{kipf2016semi,velivckovic2017graph,schlichtkrull2018modeling} are powerful graph representation learners suited for tasks such as node classification. 
GNNs update the representation $\vh_v^{(l)}$ of node $v$ at layer $l$ by aggregating messages $\vm^{(l)}_{vv'}$ from each neighbor $v'$. During KGQA, the message passing is also conditioned to the given question $q$~\citep{he2021improving}. For readability purposes, we present the following GNN update for KGQA,
\begin{equation}
    \vh_v^{(l)}= \psi \Big(\vh_v^{(l-1)}, \sum_{v' \in \mathcal{N}_v } \omega(q, r) \cdot \vm_{vv'}^{(l)} \Big),
    \label{eq:gnn-th}
\end{equation}
where function $\omega(\cdot)$ measures how relevant relation $r$ of fact $(v, r, v')$ is to question $q$. Neighbor messages $\vm_{vv'}^{(l)}$ are aggregated by a sum-operator $\sum$, which is typically employed in GNNs. Function $\psi(\cdot)$ combines representations from consecutive GNN layers.

\textbf{LLMs}. LLMs for KGQA use KG information to perform retrieval-augmented generation (RAG) as follows. The retrieved subgraph is first converted into natural language so that it can be processed by the LLM. The input given to the LLM contains the KG factual information along with the  question and a prompt. For instance, the input becomes ``\texttt{Knowledge: Jamaica $\rightarrow$ language\_spoken $\rightarrow$ English \textbackslash n Question: Which language do Jamaican people speak?}'', where the LLM has access to KG information for answering the question. 

\textbf{Landscape of KGQA methods}. \Figref{fig:retrieval-reasoning} presents the landscape of existing KGQA methods with respect to KG retrieval and reasoning. GNN-based methods, such as GraftNet~\citep{sun-etal-2018-open}, NSM~\citep{he2021improving}, and ReaRev~\citep{mavromatis2022rearev}, reason over a dense KG subgraph leveraging the GNN's ability to handle complex graph information. 
Recent LLM-based methods leverage the LLM's power for both retrieval and reasoning. ToG~\citep{sun2024tog} uses the LLM to retrieve relevant facts hop-by-hop. RoG~\citep{luo2024rog} uses the LLM to generate plausible relation paths which are then mapped on the KG to retrieve the relevant information.

\textbf{LLM-based Retriever}. We present an example of an LLM-based retriever (RoG; ~\citep{luo2024rog}). Given training question-answer pairs, RoG extracts the shortest paths to the answers starting from question entities for fine-tuning the retriever. Based on the extracted paths, an LLM (LLaMA2-Chat-7B~\citep{touvron2023llama}) is fine-tuned to generate reasoning paths given a question $q$ as 
\begin{equation}
    \text{LLM}(\text{prompt}, q) \Longrightarrow \{ r_1 \rightarrow \cdots \rightarrow  r_t \}_k,
    \label{eq:rog}
\end{equation}
where the prompt is ``\texttt{Please generate a valid relation path that can be helpful for answering the following question: \{Question\}}''.
Beam-search decoding is used to generate $k$ diverse sets of reasoning paths for better answer coverage, e.g., relations \texttt{\{<official\_language>, <language\_spoken>\}} for the question ``\texttt{Which language do Jamaican people speak?}''. The generated paths are mapped on the KG, starting from the question entities, in order to retrieve the intermediate entities for RAG, e.g.,  \texttt{<Jamaica $\rightarrow$ language\_spoken $\rightarrow$ English>}.

\section{\grag} \label{sec:gnn-rag}

We introduce \textsc{Gnn-Rag}, a novel method for combining language understanding abilities of LLMs with the reasoning abilities of GNNs in a retrieval-augmented generation (RAG) style. We provide the overall framework in \Figref{fig:gnn-rag}.  First, a GNN reasons over a dense KG subgraph to retrieve answer candidates for a given question. Second, the shortest paths in the KG  that connect question entities and GNN-based answers are extracted to represent useful KG reasoning paths. The extracted paths are verbalized and given as input for LLM reasoning with RAG. In our \textsc{Gnn-Rag} framework, the GNN acts as a dense subgraph reasoner to extract useful graph information, while the LLM leverages its natural language processing ability for ultimate KGQA.

\begin{figure*}[tb]
    \centering
    \includegraphics[width=0.9\linewidth]{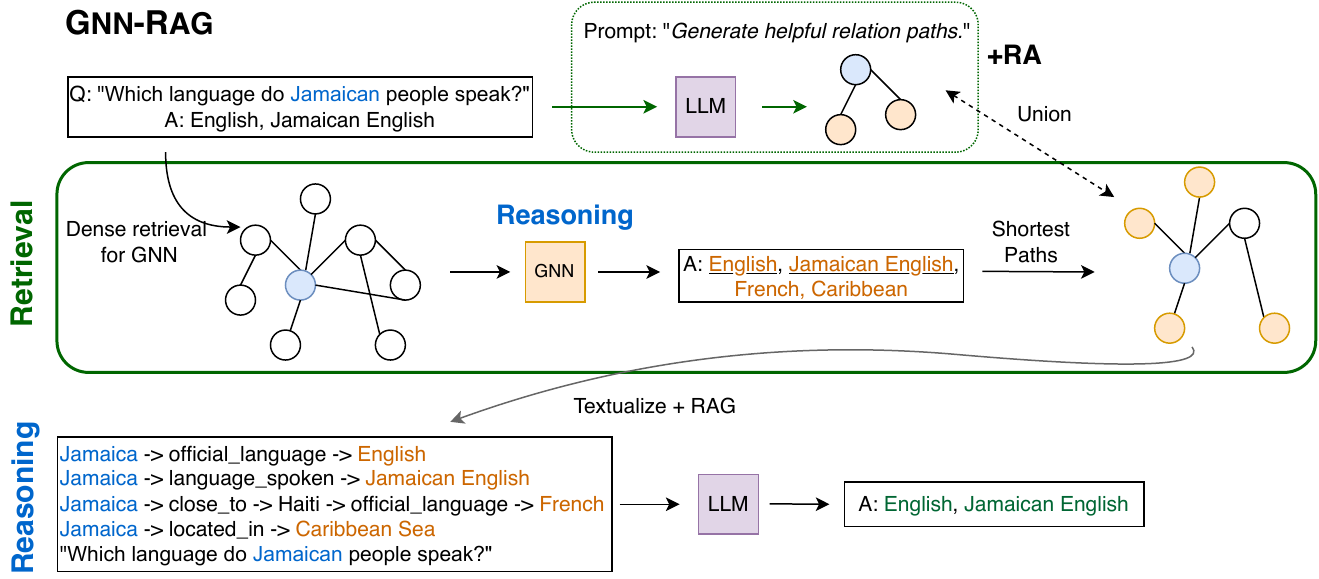}
    \caption{\grag: The GNN reasons over a dense subgraph to retrieve candidate answers, along with the corresponding reasoning paths (shortest paths from question entities to answers). The retrieved reasoning paths --optionally combined with retrieval augmentation (RA)-- are verbalized and given to the LLM for RAG.}
    \label{fig:gnn-rag}
\end{figure*}

\subsection{GNN} \label{sec:gnn}

In order to retrieve high-quality reasoning paths via \grag, we leverage state-of-the-art GNNs for KGQA. We prefer GNNs over other KGQA methods, e.g., embedding-based methods~\citep{saxena2020improving}, due to their ability to handle  complex graph interactions and answer multi-hop questions. GNNs mark themselves as good candidates for retrieval due to their architectural benefit of exploring diverse reasoning paths~\citep{mavromatis2022rearev,choi2024nutrea} that  result in high answer recall.

When GNN reasoning is completed ($L$ GNN updates via \Eqref{eq:gnn-th}), all nodes in the subgraph are scored as answers vs. non-answers based on their final GNN representations $\vh_v^{(L)}$, followed by the $\softmax(\cdot)$ operation. The GNN parameters are optimized via node classification (answers vs. non-answers) using the training question-answer pairs. During inference, the nodes with the highest probability scores, e.g., above a probability threshold, are returned as candidate answers, along with the shortest paths connecting the question entities with the candidate answers (reasoning paths). The retrieved reasoning paths are used as input for LLM-based RAG. 

Different GNNs may fetch different reasoning paths for RAG. As presented in ~\Eqref{eq:inst}, GNN reasoning depends on the question-relation matching operation $\omega(q,r)$. A common implementation of $\omega(q,r)$ is  $\phi(\vq^{(k)} \odot \vr)$~\citep{he2021improving},  where function $\phi$ is a  neural network, and $\odot$ is the element-wise multiplication.  Question representations $\vq{(k)}$ and KG relation representations $\vr$ are encoded via a shared pretrained LM~\citep{jiang2023unikgqa} as
\begin{equation}
    \vq^{(k)} = \gamma_k\big( \text{LM}(q)\big),  \; \; \vr = \gamma_c \big( \text{LM}(r) \big),
    \label{eq:inst}
\end{equation}
where $\gamma_k$ are attention-based pooling neural networks so that different representations $\vq^{(k)}$ attend to different question tokens, and $\gamma_\text{c}$ is the \texttt{[CLS]} token pooling. 

In Appendix~\ref{app:thm}, we develop a Theorem  that shows that the GNN's output depends on the question-relation matching operation $\omega(q,r)$ and as result, the choice of the LM in \Eqref{eq:inst} plays an important role regarding which answer nodes are retrieved. 
Instead of trying different GNN architectures, we employ different LMs in \Eqref{eq:inst} that result into different output node representations. Specifically, we train two separate GNN models, one using pretrained LMs, such as  SBERT~\citep{reimers2019sentence}, and one using \lmsr, a pretrained LM for question-relation matching over the KG~\citep{zhang2022subgraph}. Our experimental results suggest that, although these GNNs retrieve different KG information, they both improve RAG-based KGQA.

\subsection{LLM} \label{sec:llm}

After obtaining the reasoning paths by \grag, we verbalize them and give them as input to a downstream LLM, such as ChatGPT or LLaMA. However, LLMs are sensitive to the input prompt template and the way that the graph information is verbalized.  

To alleviate this issue, we opt to follow RAG prompt tuning~\citep{lin2023radit,zhang2024raft} for LLMs that have open weights and are feasible to train. A LLaMA2-Chat-7B model is fine-tuned based on the training question-answer pairs to generate a list of correct answers, given the prompt: \\
``\texttt{Based on the reasoning paths, please answer the given question.\textbackslash n Reasoning Paths: \{Reasoning Paths\} \textbackslash n Question: \{Question\}}''. \\
The reasoning paths are verbalised as  ``\texttt{\{question entity\} $\rightarrow$ \{relation\} $\rightarrow$ \{entity\} $\rightarrow$ $\cdots$ $\rightarrow$ \{relation\} $\rightarrow$ \{answer entity\} \textbackslash n}'' (see \Figref{fig:gnn-rag}). \\
During training, the reasoning paths are the shortest paths from question entities to answer entities. During inference, the reasoning paths are obtained by \grag. 

\subsection{Retrieval Analysis: Why GNNs \& Their Limitations} \label{sec:analysis}

\begin{wraptable}{R}{0.5\textwidth}
\vspace{-0.2in}
	\centering
	\caption{Retrieval results for WebQSP.}
	\label{tab:analysis_retrieval}%
	\resizebox{\linewidth}{!}{
	\begin{threeparttable}
		\begin{tabular}{@{}l|cc|cc@{}}
			\toprule

        \multirow{2}{*}{\textbf{Retriever}} & \multicolumn{2}{c|}{\textbf{1-hop questions}} & \multicolumn{2}{c}{\textbf{2-hop questions}} \\
            & \#Input Tok. & \%Ans. Cov. &  \#Input Tok. & \%Ans. Cov. \\       
            \midrule
            RoG~\citep{luo2024rog} & 150& \textbf{87.1} & 435 & 82.1 \\
            \midrule
            GNN ($L=1$) & 112 & 83.6 & 2,582 & 79.8 \\
             GNN ($L=3$) & 105 & 82.4 & 357& \textbf{88.5}\\
			\bottomrule
		\end{tabular}%
		\end{threeparttable}
}

\end{wraptable}%

GNNs leverage the graph structure to retrieve relevant parts of the KG that contain multi-hop information.  We provide experimental evidence on why GNNs are good retrievers for multi-hop KGQA. 
We train two different GNNs, a deep one ($L=3$) and a shallow one ($L=1$), and measure their retrieval capabilities. We report the `Answer Coverage' metric, which evaluates whether the retriever is able to fetch at least one correct answer for RAG. Note that `Answer Coverage' does not measure downstream KGQA performance but whether the retriever fetches relevant KG information. `\#Input Tokens’ denotes the median number of the input tokens of the retrieved KG paths.
Table~\ref{tab:analysis_retrieval} shows GNN retrieval results for single-hop and multi-hop questions of the WebQSP dataset compared to an LLM-based retriever (RoG; \Eqref{eq:rog}). The results indicate that deep GNNs ($L=3$) can handle the complex graph structure and retrieve useful \emph{multi-hop} information more effectively (\%Ans. Cov.) and efficiently (\#Input Tok.) than the LLM and the shallow GNN.

On the other hand, the limitation of GNNs is for  simple (1-hop) questions, where accurate question-relation matching is more important than deep graph search (see our Theorem in Appendix~\ref{app:analysis} that states this GNN limitation). In such cases, the LLM retriever is better at selecting the right KG information due to its natural language understanding abilities (we provide an example later in \Figref{fig:faith-ra}).

\subsection{Retrieval Augmentation (RA)} \label{sec:ret_aug}

Retrieval augmentation (RA) combines the retrieved KG information from different approaches to increase diversity and answer recall. Motivated by the results in \Secref{sec:analysis}, we present a RA technique (\textbf{\gragx+RA}), which complements the GNN retriever with an LLM-based retriever to combine their strengths on multi-hop and single-hop questions, respectively. Specifically, we experiment with the RoG retrieval, which is described in \Eqref{eq:rog}.  During inference, we take the union of the reasoning paths retrieved by the two retrievers.

A downside of LLM-based retrieval is that it requires multiple generations (beam-search decoding) to retrieve diverse paths, which trades efficiency for effectiveness (we provide a performance analysis in Appendix~\ref{app:analysis}). A cheaper alternative is to perform RA by combining the outputs of different GNNs, which are equipped with different LMs in \Eqref{eq:inst}. Our \textbf{\gragx+Ensemble} takes the union of the retrieved paths of the two different GNNs (GNN+SBERT \& GNN+\lmsr) as input for RAG.

\section{Experimental Setup}

\textbf{KGQA Datasets}.
We experiment with two widely used KGQA benchmarks: WebQuestionsSP (WebQSP)~\citep{yih2015semantic}, 
Complex WebQuestions 1.1 (CWQ)~\citep{talmor2018web}.
\textbf{WebQSP} contains 4,737 natural language questions that are answerable using a subset Freebase KG~\citep{bollacker2008freebase}. The questions require up to 2-hop reasoning within this KG. \textbf{CWQ} contains 34,699 total complex questions that require up to 4-hops of reasoning over the KG. We provide the detailed dataset statistics in Appendix~\ref{app:exp}.

\textbf{Implementation \& Evaluation}. 
For subgraph retrieval, we use the linked entities and the pagerank algorithm to extract dense graph information~\citep{he2021improving}. We employ ReaRev~\citep{mavromatis2022rearev}, which is a GNN targeting at \emph{deep} KG reasoning (\Secref{sec:analysis}), for \grag.  The default implementation is to combine ReaRev with SBERT as the LM in \Eqref{eq:inst}. In addition, we combine ReaRev with \lmsr, which is obtained by following the implementation of SR~\citep{zhang2022subgraph}.  We employ RoG~\citep{luo2024rog} for RAG-based prompt tuning (\Secref{sec:llm}).
For evaluation, we adopt Hit, Hits@1 (H@1), and F1 metrics. \underline{Hit} measures if \underline{any} of the true answers is found in the generated response, which is typically employed when evaluating LLMs. \underline{H@1} is the accuracy of the \underline{top/first} predicted answer. \underline{F1} takes into account the \underline{recall} (number of true answers found) and the \underline{precision} (number of false answers found) of the generated answers. 
Further experimental setup details are provided in Appendix~\ref{app:exp}.

\begin{table*}[tb]
	\centering
	\caption{Performance comparison of different methods on the two KGQA benchmarks. We denote the \textbf{best} and \underline{second-best} method.}
	\label{tab:main_results}%
	\resizebox{0.75\columnwidth}{!}{
	\begin{threeparttable}
		\begin{tabular}{@{}c|lcccccc@{}}
			\toprule
			 \multirow{2}{*}{\textbf{Type}} & \multirow{2}{*}{\textbf{Method}} & \multicolumn{3}{c}{\textbf{WebQSP}} & \multicolumn{3}{c}{\textbf{CWQ}} \\  \cmidrule(l){3-5} \cmidrule(l){6-8} 
		      & & Hit & H@1 & F1 & Hit & H@1 & F1 \\
            \midrule
            \multirow{4}{*}{Embedding} & KV-Mem~\cite{miller2016key} &	-- & 46.7 &  38.6 & -- & 21.1 & -- \\
            & EmbedKGQA~\cite{saxena2020improving} & -- & 66.6 & -- &  -- & -- & -- \\
			& TransferNet~\cite{shi2021transfernet} & -- & 71.4 & -- & -- & 48.6 & --  \\
			& Rigel~\cite{sen2021expanding} & -- & 73.3 & -- & -- & 48.7 & -- \\ \midrule
            \multirow{8}{*}{GNN} & GraftNet~\cite{sun-etal-2018-open} & -- & 66.7 & 62.4  & -- & 36.8 & 32.7   \\
            & PullNet~\cite{sun2019pullnet} & -- & 68.1 & -- &  -- & 45.9 & --\\
            & NSM~\cite{he2021improving}	&	-- & 68.7 & 62.8 & -- & 47.6 & 42.4 \\
            & SR+NSM(+E2E)~\citep{zhang2022subgraph} & -- & 69.5 & 64.1 & -- & 50.2 & 47.1 \\
            & NSM+h~\cite{he2021improving} & -- & 74.3 & 67.4 & -- & 48.8 & 44.0 \\
            & SQALER~\cite{atzeni2021sqaler} & -- & 76.1 & -- & -- & --  & -- \\
            & UniKGQA~\citep{jiang2023unikgqa} & -- &  77.2 &  72.2 & -- & 51.2 & 49.1 \\
            & ReaRev~\citep{mavromatis2022rearev} & -- & 76.4 & 70.9 & -- &  52.9 & 47.8   \\ 
            & ReaRev + \lmsr  & -- & 77.5 & \underline{72.8} & -- &  53.3 & 49.7   \\ 
            \midrule
            \multirow{5}{*}{LLM}             & Flan-T5-xl \citep{chung2024scaling}     & 31.0   & --       & --             & 14.7          & --    & --         \\
                                          & Alpaca-7B \citep{alpaca2023}         & 51.8          & --    & --          & 27.4          & -- & --             \\
                                          & LLaMA2-Chat-7B \citep{touvron2023llama} & 64.4   &      -- & --   & 34.6  &   -- & --          \\
                                          & ChatGPT                                 & 66.8          &-- & --& 39.9          &-- & --\\
                                          & ChatGPT+CoT                             & 75.6          & -- & -- & 48.9          & -- & -- \\ \midrule
            \multirow{7}{*}{KG+LLM}    
            & KD-CoT~\citep{wang2023knowledge} & 68.6 & -- & 52.5 & 55.7 & -- & -- \\
            & StructGPT~\citep{jiang2023structgpt} & 72.6 & -- & -- & -- & -- & -- \\
            & KB-BINDER~\citep{li2023binder} & 74.4 & -- & -- & -- & -- & -- \\
            & ToG+LLaMA2-70B~\citep{sun2024tog} & 68.9 & -- & -- & 57.6 & -- & -- \\
            & ToG+ChatGPT~\citep{sun2024tog}& 76.2 & -- & -- & 58.9 & -- & -- \\
            & ToG+GPT-4~\citep{sun2024tog} & 82.6 & -- & -- & \textbf{69.5} & -- & -- \\
            & RoG~\citep{luo2024rog} & \underline{85.7} & 80.0 &  70.8 & 62.6 & 57.8 & 56.2 \\ 
            \midrule
            \multirow{4}{*}{GNN+LLM} & G-Retriever~\citep{he2024gretriever}& --  & 70.1 & -- & -- & -- & -- \\
            & \grag (\textbf{Ours}) & \underline{85.7} & \underline{80.6} & 71.3 & 66.8 & \underline{61.7} & \underline{59.4}\\
            & \gragx+RA (\textbf{Ours}) & \textbf{90.7} & \textbf{82.8} & \textbf{73.5 }& \underline{68.7} & \textbf{62.8} & \textbf{60.4}\\

			\bottomrule
		\end{tabular}%
		\begin{tablenotes}
        \item Hit is used for LLM evaluation.
		\item We use the default \grag (+RA) implementation. \gragx, RoG, KD-CoT, and G-Retriever use 7B fine-tuned LLaMA2 models. KD-CoT employs ChatGPT as well.
        \end{tablenotes}
		\end{threeparttable}
}

\end{table*}%

\textbf{Competing Methods}.  We compare with SOTA GNN and LLM methods for KGQA~\citep{mavromatis2022rearev,li2023binder}. We also include earlier embedding-based methods~\citep{saxena2020improving} as well as zero-shot/few-shot LLMs~\citep{alpaca2023}. We do not compare with semantic parsing methods~\citep{yu2022decaf} as they use additional training data (SPARQL annotations), which are difficult  to obtain in practice. Furthermore, we compare \grag with LLM-based retrieval approaches~\citep{luo2024rog,sun2024tog} in terms of efficiency and effectiveness.

\section{Results}

\textbf{Main Results}.
Table~\ref{tab:main_results} presents performance results of different KGQA methods. \grag is the method that performs overall the best, achieving state-of-the-art results on the two KGQA benchmarks in almost all metrics. The results show that equipping LLMs with GNN-based retrieval boosts their reasoning ability significantly (GNN+LLM vs. KG+LLM). Specifically, \gragx+RA outperforms RoG by 5.0--6.1\% points at Hit, while it outperforms or matches ToG+GPT-4 performance, using an LLM with only 7B parameters and much fewer LLM calls -- we estimate ToG+GPT-4 has an overall cost above \$800, while \grag can be deployed on a single 24GB GPU. \gragx+RA outperforms ToG+ChatGPT by up to 14.5\% points at Hit and the best performing GNN by 5.3--9.5\% points at Hits@1 and by 0.7--10.7\% points at F1.

\captionsetup{width=\columnwidth}
\begin{table}[tb]
	\centering
 \vspace{-0.1in}
	\caption{Performance analysis (F1) on multi-hop (hops$\geq 2$) and multi-entity (entities$\geq 2$) questions.}
	\label{tab:multi}%
	\resizebox{0.5\columnwidth}{!}{
		\begin{tabular}{@{}l|cc|cc@{}}
			\toprule
            \multirow{2}{*}{Method} & \multicolumn{2}{c|}{\textbf{WebQSP}} & \multicolumn{2}{c}{\textbf{CWQ}} \\  
			 & multi-hop & multi-entity & multi-hop & multi-entity\\
            \midrule
            LLM (No RAG) & 48.4 & 61.5 & 33.7 & 32.3 \\
            RoG & 63.3 & 65.1  & 59.3 & 58.3 \\
			\midrule
            \grag &  69.8 & 82.3 & 68.2 & 64.8 \\
            \gragx+RA &  71.1 & 88.8 & 69.3 & 65.6 \\
            
			\bottomrule
		\end{tabular}%
		}
\end{table}%

\textbf{Multi-Hop \& Multi-Entity KGQA}. Table~\ref{tab:multi} compares performance results on multi-hop questions, where answers are more than one hop away from the question entities, and multi-entity questions, which have more than one question entities. \grag leverages GNNs to handle complex graph information and outperforms RoG (LLM-based retrieval) by 6.5--17.2\% points at F1 on WebQSP and by 8.5--8.9\% points at F1 on CWQ. In addition, \gragx+RA offers an additional improvement by up to 6.5\% points at F1. The results show that \grag is an effective retrieval method when deep graph search is important for successful KGQA.

\definecolor{LightCyan}{rgb}{0.9,0.95,1}
\newcolumntype{b}{>{\columncolor{LightCyan}}c}

\begin{table*}[tb]
	\centering
	\caption{Performance comparison (F1 at KGQA) of different retrieval augmentations (\Secref{sec:ret_aug}). `\#LLM Calls' are controlled by the hyperparameter $k$ (number of beams) during beam-search decoding for LLM-based retrievers,  `\#Input Tokens' denotes the median number of tokens.}
	\label{tab:retrieval}%
	\resizebox{0.9\columnwidth}{!}{
	\begin{threeparttable}
		\begin{tabular}{@{}ll|cc|b@{}}
			\toprule
			 \multirow{3}{*}{\textit{\textcolor{olive}{Retriever}}} & \multirow{3}{*}{\textit{\textcolor{blue}{KGQA Model}}} & \multicolumn{2}{c|}{\textcolor{olive}{\textit{Input/Graph Statistics}}} & \multicolumn{1}{c}{\textit{\textcolor{blue}{\textbf{KGQA Performance}}}} \\
        & & \#LLM Calls & \#Input Tokens &   \textbf{F1} (\%)\\ 
        & & & WebQSP / CWQ & WebQSP / CWQ  \\
            \midrule
            a) Dense Subgraph & \textcolor{teal}{(i)} \;GNN + SBERT (Eq.~\ref{eq:inst}) & 0 & \multirow{2}{*}{--}  & 70.9 / 47.8\\
             b) Dense Subgraph & \textcolor{violet}{(ii)} GNN + \lmsr (Eq.~\ref{eq:inst}) & 0 & &  72.8 / 49.1\\
             
            \midrule
            c) None & \multirow{4}{*}{LLaMA2-Chat-7B (tuned)} & 0 & 59 / 70 & 49.7 / 33.8\\
             d) \textcolor{brown}{(iii)} RoG (LLM-based; Eq.~\ref{eq:rog}) &  & 3 & 202 / 325  & 70.8 / 56.2\\
             
             e) \grag (\textit{default}): \textcolor{teal}{(i)} &  & 0 & 144 / 207  & 71.3 / 59.4 \\ 
             f) \gragx: \textcolor{violet}{(ii)} & & 0 & 124 / 206  & 71.5 / 58.9 \\
             \midrule
             g) \gragx+Ensemble: \textcolor{teal}{(i)} + \textcolor{violet}{(ii)} & \multirow{4}{*}{LLaMA2-Chat-7B (tuned)} & 0 &  156 / 281 &  71.7 / 57.5\\
             h) \gragx+RA (\textit{default}): \textcolor{teal}{(i)} + \textcolor{brown}{(iii)} & & 3 & 299 / 540  & \textbf{73.5} / 60.4\\
            i) \gragx+RA: \textcolor{violet}{(ii)} + \textcolor{brown}{(iii)} &  & 3 & 267 / 532  & 73.4 / \textbf{61.0} \\
            j) \gragx+All: \textcolor{teal}{(i)} + \textcolor{violet}{(ii)} + \textcolor{brown}{(iii)} &  & 3 & 330 / 668  & 72.3 / 59.1 \\

			\bottomrule
		\end{tabular}%
        
		\begin{tablenotes}
        \item For other experiments, we use the \textit{default} \grag  and \gragx+RA  implementations. The GNN used is ReaRev.
        \end{tablenotes}
		\end{threeparttable}
}

\end{table*}%

\definecolor{LightCyan}{rgb}{0.9,0.95,1}
\newcolumntype{b}{>{\columncolor{LightCyan}}c}

\textbf{Retrieval Augmentation}. Table~\ref{tab:retrieval} compares different retrieval augmentations for \grag. The primary metric is F1, while the other metrics assess how well the methods retrieve relevant information from the KG. Based on the results, we make the following conclusions: 
\begin{enumerate}
    \item  GNN-based retrieval is more efficient (\#LLM Calls, \#Input Tokens) and effective (F1) than LLM-based retrieval, especially for complex questions (CWQ); see rows (e-f)  vs. row (d). 
    \item Retrieval augmentation works the best (F1) when combining GNN-induced reasoning paths with LLM-induced reasoning paths as they fetch non-overlapping KG information (increased \#Input Tokens) that improves retrieval for KGQA; see rows (h) \& (i).
    \item Augmenting all retrieval approaches does not necessarily cause improved performance (F1) as the long input (\#Input Tokens) may confuse the LLM; see rows (g/j) vs. rows (e/h).
    \item Although the two GNNs perform differently at KGQA (F1), they both improve RAG with LLMs; see rows (a-b) vs. rows (e-f). \textit{We note though that weak GNNs are not effective retrievers (see Appendix~\ref{app:gnn-model})}.
\end{enumerate}

\begin{wraptable}{R}{0.4\textwidth} %
	\centering
    \vspace{-0.1in}
	\caption{Retrieval effect on performance (\% Hit) using various LLMs. }
	\label{tab:llms}%
	\resizebox{\linewidth}{!}{
	\begin{threeparttable}
		\begin{tabular}{lll}
			\toprule
			
			  \textbf{Method} & \textbf{WebQSP} & \textbf{CWQ}\\
            \midrule
            ChatGPT & 51.8 & 39.9 \\
            \; + ToG & 76.2 & 58.9 \\
            \; + RoG & 81.5 & 52.7 \\
            \; + \grag (+RA)& 85.3 (\textbf{87.9}) & 64.1 (\textbf{65.4}) \\
            \midrule
            Alpaca-7B & 51.8 & 27.4 \\
            \; + RoG & 73.6 & 44.0 \\
            \; + \grag (+RA)& 76.2 (\textbf{76.5}) & \textbf{54.5} (50.8)\\
            \midrule
            LLaMA2-Chat-7B & 64.4 & 34.6 \\
            \; + RoG & 84.8 & 56.4 \\
            \; + \grag (+RA)& 85.2 (\textbf{88.5}) & 62.7 (\textbf{62.9}) \\
            LLaMA2-Chat-70B & 57.4 & 39.1 \\
            \; + ToG & 68.9 & 57.6\\
            \midrule
            Flan-T5-xl & 31.0 & 14.7 \\
            \; + RoG & 67.9 & 37.8  \\
            \; + \grag (+RA)& \textbf{74.5} (72.3) & \textbf{51.0} (41.5)\\
			\bottomrule
		\end{tabular}%
		\end{threeparttable}

    }
    \vspace{-0.2in}

\end{wraptable}%

In addition, \grag improves the vanilla LLM by up to 176\% at F1 without incurring additional LLM calls; see row (c) vs. row (e). Overall,  retrieval augmentation of GNN-induced and LLM-induced paths combines their strengths and achieves the best KGQA performance.

\textbf{Retrieval Effect on LLMs}. 
Table~\ref{tab:llms} presents performance results of various LLMs using \grag or LLM-based retrievers (RoG and ToG). We report the Hit metric as it is difficult to extract the number of answers from LLM’s output. \grag (+RA) is the retrieval approach that achieves the  largest improvements for RAG. For instance, \gragx+RA improves ChatGPT by up to 6.5\% points at Hit over RoG and ToG. Moreover, \grag substantially improves the KGQA performance of weaker LLMs, such as Alpaca-7B and Flan-T5-xl. The improvement over RoG is up to 13.2\% points at Hit, while \grag outperforms LLaMA2-Chat-70B+ToG using a lightweight 7B LLaMA2 model. The results demonstrate that \grag can be integrated with other LLMs to improve their KGQA reasoning without retraining.

\textbf{Case Studies on Faithfulness}. \Figref{fig:faith} illustrates two case studies from the CWQ dataset, showing how \grag improves LLM's faithfulness, i.e., how well the LLM follows the question's instructions and uses the right information from the KG. In both cases, \grag retrieves multi-hop information, which is necessary for answering the questions correctly. In the first case, \grag retrieves both crucial facts \texttt{<Gilfoyle $\rightarrow$ characters\_that\_have\_lived\_here $\rightarrow$ Toronto>} and  \texttt{<Toronto $\rightarrow$ province.capital $\rightarrow$ Ontario>} that are required to answer the question, unlike the KG-RAG baseline (RoG) that fetches only the first fact. In the second case, the KG-RAG baseline incorrectly retrieves  information about \texttt{<Erin Brockovich $\rightarrow$ person>} and not \texttt{<Erin Brockovich $\rightarrow$ film\_character>} that the question refers to. \grag uses GNNs to explore how \texttt{<Erin Brockovich>} and \texttt{<Michael Renault Mageau>} entities are related in the KG, resulting into retrieving facts about \texttt{<Erin Brockovich $\rightarrow$ film\_character>}. The retrieved facts include important information \texttt{<films\_with\_this\_crew\_job $\rightarrow$ Consultant>}.

\Figref{fig:faith-ra} illustrates one case study from the WebQSP dataset, showing how RA (\Secref{sec:ret_aug}) improves \grag.  Initially, the GNN does not retrieve helpful information due to its limitation to understand natural language, i.e., that  \texttt{<jurisdiction.bodies>} usually ``\texttt{make the laws}''. \gragx+RA retrieves the right information, helping the LLM answer the question correctly.

Further ablation studies are provided in Appendix~\ref{app:exp-res}. Limitations are discussed in Appendix~\ref{app:limit}.

\begin{figure*}[tb]
    \centering
    \includegraphics[width=0.95\linewidth]{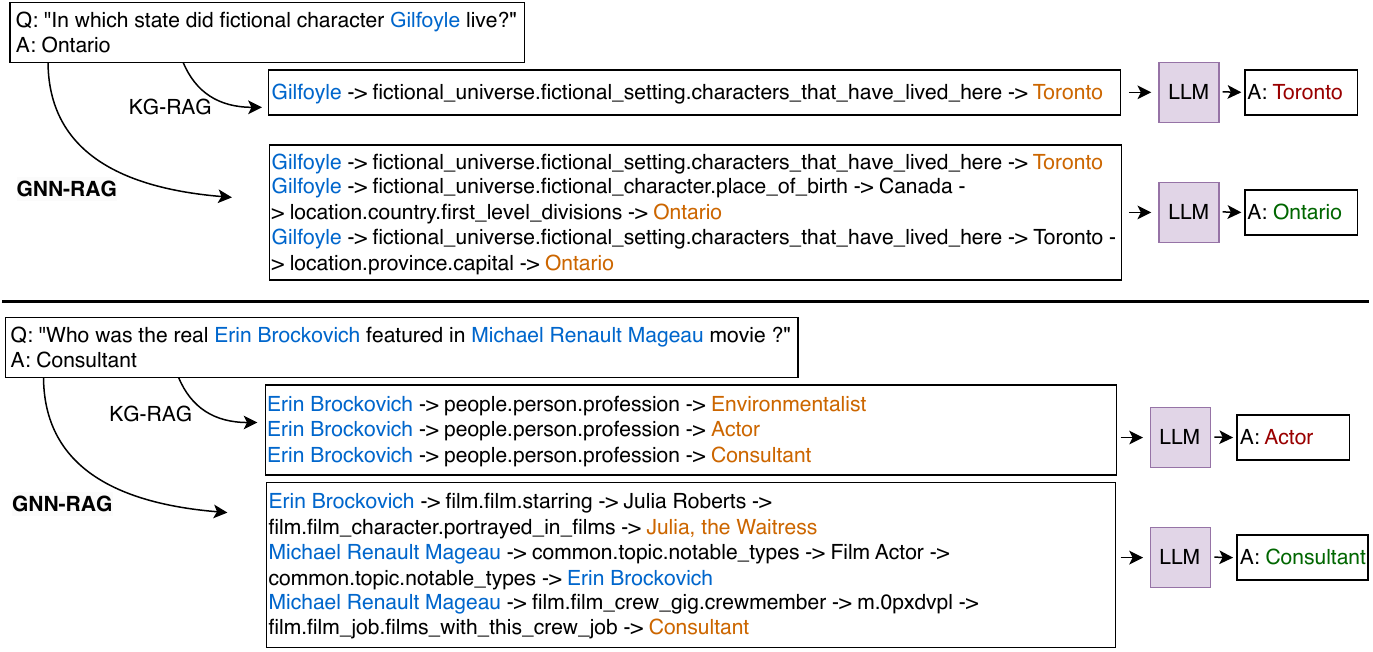}
    \caption{Two case studies that illustrate how \grag improves the LLM's faithfulness. In both cases, \grag retrieves \emph{multi-hop} information that is necessary for answering  the complex questions.}
    \label{fig:faith}
\end{figure*}

\begin{figure*}[tb]
    \centering
    \includegraphics[width=0.95\linewidth]{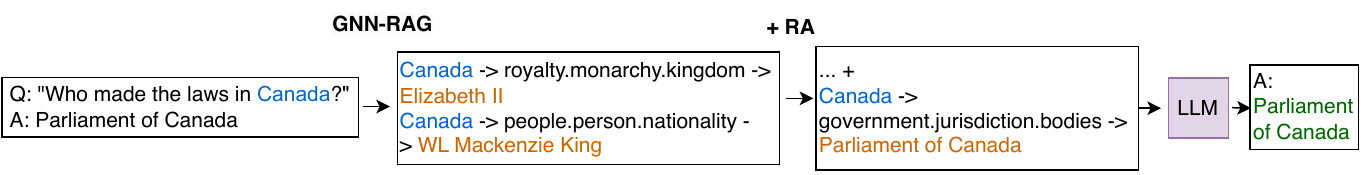}
    \caption{One case study that illustrates the benefit of retrieval augmentation (RA). RA uses LLMs to fetch semantically relevant KG information, which may have been missed by the GNN.}
    \label{fig:faith-ra}
\end{figure*}

\section{Conclusion}

We introduce \grag, a novel method for combining the reasoning abilities of LLMs and
GNNs for RAG-based KGQA. Our \textbf{contributions} are the following. (1) \textbf{Framework}: \grag repurposes GNNs for KGQA retrieval to enhance the reasoning
abilities of LLMs. Moreover, our retrieval analysis guides the design of a
retrieval augmentation technique to boost \grag performance. (2) \textbf{Effectiveness \& Faithfulness}: \grag achieves state-of-the-art performance in two
widely used KGQA benchmarks (WebQSP and CWQ). Furthermore, \grag is shown to retrieve multi-hop information that is necessary for faithful LLM reasoning on complex questions. (3) \textbf{Efficiency}: \grag improves vanilla LLMs on KGQA performance without incurring additional LLM calls as existing RAG systems for KGQA require. In addition, \grag outperforms or matches GPT-4 performance with a 7B tuned LLM.

\newpage

\bibliography{GNN-RAG}
\bibliographystyle{plainnat}

\newpage
\appendix

\textbf{\Large Appendix / supplemental material}

\section{Analysis} \label{app:analysis}

\begin{wraptable}{R}{0.4\textwidth}
\vspace{-0.2in}
	\centering
	\caption{Efficiency vs. effectiveness trade-off of LLM-based retrieval.}
	\label{tab:analysis_retrieval_app}%
	\resizebox{\linewidth}{!}{
	\begin{threeparttable}
		\begin{tabular}{@{}l|cc@{}}
			\toprule
			 \multirow{2}{*}{\textbf{Retrieval}} & 
            \#LLM Calls & Answer Hit (\%) \\
            & (efficiency) & (effectiveness) \\
            
            \midrule
            \multirow{2}{*}{RoG~\citep{luo2024rog}} & 3 & 85.7 \\
            & 1 & \textcolor{red}{77.2}\\
            \midrule
            \multirow{2}{*}{ToG~\citep{sun2024tog}} & \textcolor{red}{up to 21} & 76.2 \\
            & 3 & \textcolor{red}{66.3} \\
            \midrule 
            \grag  & \textcolor{teal}{\textbf{0}} & \textcolor{teal}{\textbf{87.2}} \\
			
			\bottomrule
		\end{tabular}%
		\begin{tablenotes}
        \item \#LLM Calls are controlled by the hyperparameter $k$ (number of beams) during beam-search decoding.
        \end{tablenotes}
		\end{threeparttable}
}

\vspace{-0.2in}
\end{wraptable}%

In this section, we analyze the reasoning and retrieval abilities of GNN and LLMs, respectively.

\begin{definition}[Ground-truth Subgraph] \label{def:1}
Given a question $q$, we define its ground-truth reasoning subgraph $\gG_q^\ast$ as the union of the ground-truth reasoning paths that lead to the correct answers $\{a\}$. Reasoning paths are defined as the KG paths that reach the answer nodes, starting from the question entities $\{e\}$, e.g., \texttt{<Jamaica $\rightarrow$ language\_spoken $\rightarrow$ English>} for question ``\texttt{Which language do Jamaican people speak?}''. In essence, $\gG_q^\ast$ contains only the necessary entities and relations that are needed to answer $q$. 
\end{definition}

\begin{definition}[Effective Reasoning]\label{def:1a}
We define that a model $M$ \emph{reasons effectively}  if its output is $\{a\} = M(\gG_q^\ast, q)$, i.e., the model returns the correct answers given the ground-truth subgraph $\gG_q^\ast$.
\end{definition}

As KGQA methods do not use the ground-truth subgraph $\gG_q^\ast$ for reasoning, but the retrieved subgraph $\gG_q$, we identify two cases in which the reasoning model \emph{cannot} reason effectively, i.e., $\{a\} \neq M(\gG_q, q)$.

\underline{Case 1}: $\gG_q \subset \gG_q^\ast $, i.e., the retrieved subgraph $\gG_q$ does not contain all the necessary information for answering $q$. An application of this case is when we use LLMs for retrieval. As LLMs are not designed to handle complex graph information, the retrieved subgraph $\gG_q$ may contain incomplete KG information. Existing LLM-based methods rely on employing an increased number of LLM calls (beam search decoding) to fetch diverse reasoning paths that approximate $\gG_q^\ast$. Table~\ref{tab:analysis_retrieval_app} provides experimental evidence  that shows how LLM-based retrieval trades computational efficiency for effectiveness. \emph{In particular, when we switch from beam-search decoding to greedy decoding for faster LLM retrieval, the KGQA performance drops by  8.3--9.9\% points at answer hit}.

\underline{Case 2}: $\gG_q^\ast  \subset \gG_q$ and model $M$ cannot “filter-out” irrelevant facts during reasoning. An application of this case is when we use GNNs for reasoning. GNNs cannot understand the textual semantics of KGs and natural questions the same way as LLMs do, and they reason ineffectively if they cannot tell the irrelevant KG information. We develop the following Theorem that supports this case for GNNs.

\begin{theorem}[Simplified]
\label{thm:1}
Under mild assumptions and due to the sum operator of GNNs in \Eqref{eq:gnn-th}, a GNN can reason effectively  by selecting question-relevant facts and filtering-out question-irrelevant facts through $\omega(q, r)$. 
\end{theorem}
We provide the full theorem and its proof in Appendix~\ref{app:thm}. Theorem~\ref{thm:1} suggests that GNNs need to perform semantic matching via function
$\omega(q, r)$
apart from leveraging the graph information encoded in the KG. Our analysis suggests that  GNNs lack reasoning abilities for KGQA if they cannot perform effective semantic matching between the KG and the question.

\section{Full Theorem \& Proof} \label{app:thm}

To analyze under which conditions GNN perform well for KGQA, we  use the ground-truth subgraph $\mathcal{G}_q^\ast$ for a question $q$, as defined in Definition~\ref{def:1}. We  compare the output representations of a GNN over the ground-truth $\mathcal{G}_q^\ast$ and another $\mathcal{G}_q$ to measure how close the two outputs are.

We always assume $\gG_q^\ast \subseteq \gG_q$ for a question $q$. 1-hop facts that contain $v$ are denoted as $\gN_v^\ast$.

\begin{definition} \label{def:2}
Let $M$ be a GNN model for answering question $q$ over a KG $\mathcal{G}_q$, where the output is computed by $M(q,\gG_q)$. $M$ consists of $L$ reasoning steps (GNN layers).  We assume $M$ is an effective reasoner, according to Definition~\ref{def:1a}. Furthermore, we define the reasoning process $\gR_{M, q, \gG_q}$  as the sequence of the derived node representations at each step $l$, i.e.,  
\begin{equation}
    \gR_{M, q, \gG_q} = \Big\{ \{\vh_v^{(1)}: v \in \gG_q\}, \dots, \{\vh_v^{(L)}: v \in \gG_q\} \Big\}.
    \label{eq:resonproc}
\end{equation}
We also define the optimal reasoning process for answering question $q$ with GNN $M$ as $\gR_{M, q, \gG_q^\ast}$.
We assume that zero node representations do not contribute in \Eqref{eq:resonproc}. 
\end{definition}

\begin{lemma} \label{lemma}
If two subgraphs $\gG_1$ and $\gG_2$ have the same nodes, and a GNN  outputs the same node representations for all nodes $v \in \gG_1$ and $v \in \gG_2$ at each step $l$, then the reasoning processes $\gR_{M, q, \gG_1}$ and $\gR_{M,  q, \gG_2}$ are identical. 
\end{lemma}
\noindent
This is true as $\vh_v^{(l)}$ with $l = 1, \dots, L$ for both $\gG_1$ and $\gG_2$ and by using Definition~\ref{def:2} to show $\gR_{M, q, \gG_1} = \gR_{M,  q, \gG_2}$. Note that Lemma~\ref{lemma} does not make any assumptions about the actual edges of $\gG_1$ and $\gG_2$.

To analyze the importance of semantic matching for GNNs, we consider the following GNN update
\begin{equation}
    \vh_v^{(l)}= \psi \Big(\vh_v^{(l-1)}, \sum_{v' \in \mathcal{N}_v } \omega(q, r) \cdot \vm_{vv'}^{(l)} \Big).
    \label{eq:gnn-th-a}
\end{equation}
where $\omega(\cdot, \cdot): \sR^d \times \sR^d \xrightarrow[]{} \{0,1\}$ is a binary function that decides if fact $(v,r,v')$ is relevant to question $q$ or not. Neighbor messages $\vm_{vv'}^{(l)}$ are aggregated by a sum-operator, which is typically employed in GNNs. Function $\psi(\cdot)$ combines representations among consecutive GNN layers. We assume $\vh_v^{(0)} \in \sR^d$ and that $\psi \Big(h_v^{(0)}, 0^d \Big) = 0^d$

\begin{theorem}\label{thma-1}
If $\omega(q, r)=0 , \forall (v,r,v') \notin \gG_q^\ast$ and  $\omega(q, r) = 1 , \forall (v,r,v') \in \gG_q^\ast$, then  $\gR_{M, q, \gG_q}$ is an optimal reasoning process of GNN $M$ for answering $q$. 
\end{theorem}

\begin{proof}
 We show that 
\begin{equation}
    \sum_{v' \in \mathcal{N}_v } \omega(q, r) \cdot \vm_{vv'}^{(l)}  = \sum_{v' \in \gN^\ast_v }  \vm_{vv'}^{(l)},
\end{equation}
which gives that $\gR_{M, q, \gG_q} = \gR_{M, q, \gG_q^\ast}$ via Lemma~\ref{lemma}.
This is true if 
\begin{equation}
    \omega(q, r) =
    \begin{cases}
      1 & \text{if $(v, r, v') \in \gN^\ast_v$,}\\
      0 & \text{if $(v, r, v') \notin \gN^\ast_v$,}
    \end{cases}       
\end{equation}
which means that GNNs need to filter-out question irrelevant facts. 
We consider two cases.

\textbf{Case 1}. Let $u$ denote a node that is present in $\gG_q$, but not in $\gG^\ast_q$. Then, all facts that contain $u$ are not present in $\gG^\ast_q$. Condition $\omega(q, r)=0 , \forall (v,r,v') \notin \gG_q^\ast$ of Theorem~\ref{thma-1} gives that 
\begin{align}
    \omega(q,r) &= 0, \forall (u,r,v'), \text{ and } \nonumber \\
    \omega(q,r) &= 0, \forall (v,r,u).
\end{align}
as node $u \notin \gG_q^\ast$. 
This gives
\begin{equation}
    \sum_{v' \in \mathcal{N}(u) } \omega(q, r) \cdot \; \vm_{uv'}^{(l)} = 0,
\end{equation}
as no edges will contribute to the GNN update. With $\psi \Big(h_v^{(0)}, 0^d \Big) = 0^d$, we have 
\begin{equation}
    h_u^{(l)} = 0^d, \forall u \notin \gG^\ast_q \text{ with } l=\{1, \dots, L\},
\end{equation}
which means that nodes $u \notin \gG^\ast_q $ do not contribute to the reasoning process $\gR_{M, q, \gG_q}$; see Definition~\ref{def:2}.

\textbf{Case 2}. Let $p$ denote a relation between two nodes $v$ and $v'$ that is present in $\gG_q$,  but not in $\gG^\ast_q$. We decompose the GNN update to 
\begin{equation}
    \sum_{v' \in \mathcal{N}_r(v) } \omega(q, r) \cdot \; \vm_{vv'}^{(l)} + \sum_{v' \in \mathcal{N}_p(v) } \omega(q, p) \cdot \; \vm_{vv'}^{(l)},
\end{equation}
where the first term includes facts $\mathcal{N}_r$ that are present in $\gG_q^\ast$ and the second term includes facts $\mathcal{N}_p$ that are present in $\gG_q$ only. Using the condition $\omega(q, r)=0 , \forall (v,r,v') \notin \gG_q^\ast$ of Theorem~\ref{thma-1}, we have 
\begin{equation}
    \sum_{v' \in \mathcal{N}_p(v) } \omega(q, p) \cdot \; \vm_{vv'}^{(l)} =0.
\end{equation}
Using condition $\omega(q, r) = 1 , \forall (v,r,v') \in \gG_q^\ast$, we have 
\begin{equation}
 \sum_{v' \in \mathcal{N}_r(v) } \omega(q, r) \cdot \; \vm_{vv'}^{(l)} = \sum_{v' \in \mathcal{N}_r(v) } \vm_{vv'}^{(l)}.
\end{equation}
Combining the two above expression gives 
\begin{equation}
    \sum_{v' \in \mathcal{N}_v } \omega(q, r) \cdot \vm_{vv'}^{(l)}  = \sum_{v' \in \mathcal{N}_r(v) } \vm_{vv'}^{(l)} = \sum_{v' \in \gN^\ast_v }  \vm_{vv'}^{(l)}.
\end{equation}
It is straightforward to obtain $\gR_{M, q, \gG_q} = \gR_{M, q, \gG_q^\ast}$ via Lemma~\ref{lemma} in this case.

\textbf{Putting it altogether}. Combining Case 1 and Case 2,nodes $u \notin \gG_q^\ast$ do not contribute to  $\gR_{M, q, \gG_q}$, while for other nodes we have $\gR_{M, q, \gG_q} = \gR_{M, q, \gG_q^\ast}$. Thus, overall we have $\gR_{M, q, \gG_q} = \gR_{M, q, \gG_q^\ast}$.
\end{proof}

\section{Experimental Setup} \label{app:exp}

\textbf{KGQA Datasets}.
We experiment with two widely used KGQA benchmarks: WebQuestionsSP (WebQSP)~\cite{yih2015semantic}, 
Complex WebQuestions 1.1 (CWQ)~\cite{talmor2018web}.
We also experiment with MetaQA-3~\cite{zhang2018variational} dataset.
We provide the  dataset statistics Table~\ref{tab:data}.
\textbf{WebQSP} contains 4,737 natural language questions that are answerable using a subset Freebase KG~\citep{bollacker2008freebase}. This  KG contains 164.6 million facts and 24.9 million entities. The questions require up to 2-hop reasoning within this KG. Specifically, the model needs to aggregate over two KG facts for 30\% of the questions, to reason over constraints for 7\% of the questions, and to use a single KG fact for the rest of the questions.
\noindent
\textbf{CWQ} is generated from WebQSP by extending the question entities or
adding constraints to answers, in order to construct  more complex multi-hop questions (34,689 in total). There are
four types of questions: composition (45\%), conjunction (45\%), comparative (5\%), and superlative
(5\%). The questions require up to 4-hops of reasoning over the KG, which is the same KG as in WebQSP. \textbf{MetaQA-3} consists of more than 100k 3-hop questions in the domain of movies. The questions were constructed using the KG provided by the WikiMovies~\cite{miller2016key} dataset, with about 43k entities and 135k triples. For MetaQA-3, we use 1,000 (1\%) of the training questions.

\captionsetup{width=0.7\columnwidth}
\begin{table}[tb]
	\centering
	\caption{Datasets statistics. ``avg.$|\gV_q|$'' denotes average number of entities in subgraph, and “coverage” denotes the ratio of at least one answer in subgraph.}
	\label{tab:data}%
	\resizebox{0.7\columnwidth}{!}{
		\begin{tabular}{@{}l|rrr|rr@{}}
			\toprule
			Datasets & Train & Dev  & Test & avg. $|\gV_q|$ & coverage (\%) \\
			
			\midrule
			WebQSP & 2,848 & 250 & 1,639 & 1,429.8 & 94.9\\
			CWQ &  27,639 & 3,519 & 3,531 & 1,305.8 & 79.3 \\
			MetaQA-3 &  114,196 & 14,274 & 14,274 & 497.9 & 99.0 \\
			\bottomrule
		\end{tabular}%
		}
\end{table}%

\textbf{Implementation}. 
For subgraph retrieval, we use the linked entities to the KG provided by~\cite{yih2015semantic} for WebQSP, by~\cite{talmor2018web} for CWQ. We obtain dense subgraphs by ~\cite{he2021improving}. It runs the PageRank Nibble~\cite{andersen2006local} (PRN) method starting from the linked entities to select the top-$m$ ($m=2,000$) entities to be included in the subgraph.

We employ ReaRev\footnote{https://github.com/cmavro/ReaRev\_KGQA}~\citep{mavromatis2022rearev} for GNN reasoning (\Secref{sec:gnn}) and RoG\footnote{https://github.com/RManLuo/reasoning-on-graphs}~\citep{luo2024rog} for RAG-based prompt tuning (\Secref{sec:llm}), following their official implementation codes. 
In addition, we empower ReaRev with \lmsr (\Secref{sec:gnn}), which is obtained by following the implementation of SR\footnote{https://github.com/RUCKBReasoning/SubgraphRetrievalKBQA}~\citep{zhang2022subgraph}. For both training and inference of these methods, we use their suggested hyperparameters, without performing further hyperparameter search. Model selection is performed based on the validation data. Experiments with GNNs were performed on a  Nvidia Geforce
RTX-3090 GPU over 128GB RAM
machine. Experiments with LLMs were performed on 4 A100 GPUs connected via NVLink and 512 GB of memory. The experiments are implemented with PyTorch. 

For LLM prompting during retrieval (\Secref{sec:ret_aug}), we use the following prompt:
\begin{center}
\fbox{\parbox{5.5in}{\raggedright
\texttt{Please generate a valid relation path that can be helpful for answering the following question: \\ \{Question\}}}}
\end{center}

For LLM prompting during reasoning (\Secref{sec:llm}), we use the following prompt:
\begin{center}
\fbox{\parbox{5.5in}{\raggedright
\texttt{Based on the reasoning paths, please answer the given question. Please keep the answer as simple as possible and return all the possible answers as a list.\textbackslash n \\
Reasoning Paths: \{Reasoning Paths\} \textbackslash n \\
Question: \{Question\}}}}
\end{center}

During GNN inference, each node in the subgraph is assigned a probability of being the correct answer, which is normalized via $\softmax$. To retrieve answer candidates, we sort the nodes based on the their probability scores, and select the top nodes whose cumulative probability score is below a threshold. We set the threshold to 0.95. To retrieve the shortest paths between the question entities and answer candidates for RAG, we use the NetworkX library\footnote{https://networkx.org/}.

\textbf{Competing Approaches}. 

We evaluate the following categories of methods: 1. Embedding, 2. GNN, 3. LLM, 4. KG+LMM, and 5. GNN+LLM.
\begin{enumerate}
    \item KV-Mem~\cite{miller2016key} is a key-value memory network  for KGQA. EmbedKGQA~\cite{saxena2020improving} utilizes KG pre-trained embeddings~\cite{trouillon2016complex} to improve multi-hop reasoning. TransferNet~\cite{shi2021transfernet} improves multi-hop reasoning over the relation set. Rigel~\cite{sen2021expanding} improves reasoning with questions of multiple entities.
    \item GraftNet~\cite{sun-etal-2018-open} uses a convolution-based GNN~\cite{kipf2016semi}. PullNet~\cite{sun2019pullnet} is built on top of GraftNet, but learns which nodes to retrieve via selecting shortest paths to the answers. NSM~\cite{he2021improving} is the adaptation of GNNs  for KGQA. NSM+h~\cite{he2021improving} improves NSM for multi-hop reasoning. SQALER~\cite{atzeni2021sqaler} learns which relations (facts) to retrieve during KGQA for GNN reasoning. Similarly, SR+NSM~\citep{zhang2022subgraph} proposes a relation-path retrieval. UniKGQA~\citep{jiang2023unikgqa} unifies the graph retrieval and reasoning process with a single LM. ReaRev~\citep{mavromatis2022rearev} explores diverse reasoning paths in a multi-stage manner.
    \item We experiment with instruction-tuned LLMs. Flan-T5~\citep{chung2024scaling} is based on T5, while Aplaca~\citep{alpaca2023} and LLaMA2-Chat~\citep{touvron2023llama} are based on LLaMA. ChatGPT\footnote{https://openai.com/blog/chatgpt} is a powerful closed-source LLM that excels in many complex tasks. ChatGPT+CoT uses the chain-of-thought~\citep{wei2022chain} prompt to improve the ChatGPT. We access ChatGPT \texttt{`gpt-3.5-turbo'} through its API (as of May 2024).
    \item KD-CoT~\citep{wang2023knowledge} enhances CoT prompting for LLMs with relevant knowledge from KGs. StructGPT~\citep{jiang2023structgpt} retrieves KG facts for RAG. KB-BINDER~\citep{li2023binder} enhances LLM reasoning by generating logical forms of the questions. ToG~\citep{sun2024tog} uses a powerful LLM to select  relevant facts hop-by-hop. RoG~\citep{luo2024rog} uses the LLM to generate relation paths for better planning.
    \item G-Retriever~\citep{he2024gretriever} augments LLMs with GNN-based prompt tuning.
\end{enumerate}

\section{Additional Experimental Results} \label{app:exp-res}

\subsection{Question Analysis} \label{app:question}

\captionsetup{width=\columnwidth}
\begin{table}[h]
	\centering
	\caption{Performance analysis (F1) based on the number of maximum hops that connect question entities to answer entities.}
	\label{tab:hop}%
	\resizebox{0.7\columnwidth}{!}{
		\begin{tabular}{@{}l|ccc|ccc@{}}
			\toprule
            \multirow{2}{*}{Method} & \multicolumn{3}{c|}{\textbf{WebQSP}} & \multicolumn{3}{c}{\textbf{CWQ}} \\  
			 & 1 hop & 2 hop & $\geq$3 hop & 1 hop & 2 hop & $\geq$3 hop  \\
            \midrule
            RoG & 73.4 & 63.3 & -- & 50.4 & 60.7 & 40.0\\
			\midrule
            \grag & 72.0 & 69.8 & -- & 47.4 & 69.4  & 51.8\\
            \grag+RA & 74.6 & 71.1 & -- & 48.2 & 70.9 &  47.7\\
            
			\bottomrule
		\end{tabular}%
		}
\end{table}%

Following the case studies presented in \Figref{fig:faith} and \Figref{fig:faith-ra}, we provide numerical results on how \grag improves multi-hop question answering and how retrieval augmentation (RA) enhances simple hop questions. Table~\ref{tab:hop} summarizes these results. \grag improves performance on multi-hop questions ($\geq$2 hops) by 6.5--11.8\% F1 points over RoG. Furthermore, RA improves performance on single-hop questions by 0.8--2.6\% F1 points over \grag. 

\captionsetup{width=\columnwidth}
\begin{table}[h]
	\centering
	\caption{Performance analysis (F1) based on the number of answers (\#Ans).}
	\label{tab:answer_num}%
	\resizebox{\columnwidth}{!}{
		\begin{tabular}{@{}l|cccc|cccc@{}}
			\toprule
            \multirow{2}{*}{Method} & \multicolumn{4}{c|}{\textbf{WebQSP}} & \multicolumn{4}{c}{\textbf{CWQ}} \\ \cmidrule(l){2-5 } \cmidrule(l){6-9 } 
			 & \#Ans=1 & 2$\leq$\#Ans$\leq$4 & 5$\leq$\#Ans$\leq$9 & \#Ans$\geq$10  & \#Ans=1 & 2$\leq$\#Ans$\leq$4 & 5$\leq$\#Ans$\leq$9 & \#Ans$\geq$10  \\
            \midrule
            RoG & 67.89 & 79.39 & 75.04 & 58.33 & 56.90 & 53.73 &  58.36 &  43.62 \\
			\midrule
            \grag & 71.24 & 76.30 & 74.06 & 56.28 & 60.40 & 55.52 & 61.49 & 50.08\\
            \grag+RA & 71.16 & 82.31 & 77.78 & 57.71 & 62.09 & 56.47 & 62.87 & 50.33\\
            
			\bottomrule
		\end{tabular}%
		}
\end{table}%

Table~\ref{tab:answer_num} presents results with respect to the number of correct answers. As shown, RA enhances \grag in almost all cases as it can fetch correct answers that might have been missed by the GNN.

\subsection{GNN Effect}\label{app:gnn-model}

\begin{table*}[h]
	\centering
	\caption{Performance comparison of different GNN models at complex KGQA (CWQ).}
	\label{tab:gnn_model}%
	\resizebox{0.6\columnwidth}{!}{
	\begin{threeparttable}
		\begin{tabular}{@{}ll|ccc@{}}
			\toprule
			 \multirow{2}{*}{\textbf{\textcolor{olive}{Retriever}}} & \multirow{2}{*}{\textbf{\textcolor{blue}{KGQA Model}}} &  \multicolumn{3}{c}{\textbf{CWQ}} \\  \cmidrule(l){3-5} 
		      & & Hit$^\ast$ & H@1 & F1  \\
            \midrule
            
            Dense Subgraph &  GraftNet  &  -- & 45.3 & 35.8  \\
            Dense Subgraph &  NSM  & -- & 47.9  & 42.0  \\
            Dense Subgraph &  ReaRev  & -- & 52.7 & 49.1  \\
            \midrule
            RoG & \multirow{4}{*}{LLaMA2-Chat-7B (tuned)} & 62.6 & 57.8  & 56.2 \\
            \grag: GraftNet &   & 58.2 & 51.9 & 49.4\\
            \grag: NSM & & 58.5 & 52.5 & 50.1 \\
            \grag: ReaRev &  & 66.8 & 61.7 & 59.4\\
			\bottomrule
		\end{tabular}%
        
		\end{threeparttable}
}

\end{table*}%

\begin{wraptable}{R}{0.3\textwidth} %
	\centering
 \vspace{-0.2in}
	\caption{Results on MetaQA-3 dataset.}
	\label{tab:metaqa}%
	\resizebox{\linewidth}{!}{
		\begin{tabular}{@{}l|c@{}}
			\toprule
			Method & \textbf{MetaQA-3} \\
            \midrule
            & Hit@1 \\
			\midrule
            RoG & 84.8 \\
            RoG+pretraining & 88.9 \\
            \grag & \textbf{98.6} \\
			\bottomrule
		\end{tabular}%
		}
\end{wraptable}%

\grag employs ReaRev~\citep{mavromatis2022rearev} as its GNN retriever, which is a powerful GNN for deep KG reasoning. In this section, we ablate on the impact of the GNN used for retrieval, i.e., how strong and weak GNNs affect KGQA performance. We experiment with GraftNet~\citep{sun-etal-2018-open} and NSM~\citep{he2021improving} GNNs, which are less powerful than ReaRev at KGQA.
The results are presented in Table~\ref{tab:gnn_model}. As shown, strong GNNs (ReaRev) are required in order to improve RAG at KGQA. Retrieval with weak GNNs (NSM and GraftNet) underperfoms retrieval with ReaRev by 9.2--9.8\%  and retrieval with RoG by 5.3--5.9\% points at H@1.

\subsection{MetaQA-3}

Table~\ref{tab:metaqa} presents results on the MetaQA-3 dataset, which requires 3-hop reasoning. RoG is a LLM-based retrieval which cannot handle the multi-hop KG information effectively and as a result, RoG underperforms (even when using additional pretraining data from WebQSP \& CWQ datasets). On the other hand, \grag relies on GNNs that are able to retrieve useful graph information for multi-hop questions and achieves an almost perfect performance of 98.6\% at Hit@1.

\subsection{Retrieval Augmentation}

\begin{table*}[h]
	\centering
	\caption{Performance comparison of retrieval augmentation approaches (extended).}
	\label{tab:all_retrieval}%
	\resizebox{\columnwidth}{!}{
	\begin{threeparttable}
		\begin{tabular}{@{}ll|cccccccc@{}}
			\toprule
			 \multirow{2}{*}{\textbf{\textcolor{olive}{Retriever}}} & \multirow{2}{*}{\textbf{\textcolor{blue}{KGQA Model}}} & \multirow{2}{*}{\textbf{\#LLM Calls}} & \multicolumn{3}{c}{\textbf{WebQSP}} & \multicolumn{3}{c}{\textbf{CWQ}} &\multirow{2}{*}{\textbf{Avg.}}\\  \cmidrule(l){4-6} \cmidrule(l){7-9} 
		      & & (total) & Hit$^\ast$ & H@1 & F1 & Hit$^\ast$ & H@1 & F1 \\
            \midrule
            \multirow{2}{*}{Dense Subgraph} & \textcolor{teal}{(i)} \;ReaRev + SBERT & 0 & -- & 76.4 & 70.9 & -- & 52.9 & 47.8 & --\\
             & \textcolor{violet}{(ii)} ReaRev + \lmsr & 0 & -- & 77.5 & 72.8 & -- & 52.7 & 49.1 & -- \\
             
            \midrule
            None & \multirow{4}{*}{LLaMA2-Chat-7B (tuned)} & 1 & 65.6 & 60.4 & 49.7 & 40.1 & 36.2 & 33.8 & 47.63 \\
             \textcolor{brown}{(iii)} LLM-based &  & 4 & 85.7 & 80.0 & 70.8 & 62.6 & 57.8 & 56.2 & 68.85\\
             \grag: \textcolor{teal}{(i)} & & 1 & 85.7 & 80.6   & 71.3 & 66.8 & 61.7 & 59.4 & 70.92\\ 
             \grag: \textcolor{violet}{(ii)} & & 1 &85.0 & 80.3 & 71.5 & 66.2 & 61.3 & 58.9 & 70.50 \\
             \midrule
             \grag: \textcolor{teal}{(i)} + \textcolor{violet}{(ii)} & \multirow{4}{*}{LLaMA2-Chat-7B (tuned)} & 1 & 87.2 & 81.0 & 71.7 & 65.5 & 59.5 & 57.5 & 70.40\\
             
             \grag: \textcolor{teal}{(i)} + \textcolor{brown}{(iii)} & & 4 & \textbf{90.7} & \textbf{82.8} & \textbf{73.5} &\textbf{68.7} & 62.8 & 60.4 & \textbf{73.15}\\
            \grag: \textcolor{violet}{(ii)} + \textcolor{brown}{(iii)} &  & 4 & 89.9 & 82.4 & 73.4 & 67.9 & \textbf{63.0} & \textbf{61.0} & 72.93\\
            \grag: \textcolor{teal}{(i)} + \textcolor{violet}{(ii)} + \textcolor{brown}{(iii)} &  & 4 & 90.1 & 81.7 & 72.3 & 67.3 & 61.5 & 59.1 & 72.00\\

            \midrule
            None & \multirow{3}{*}{LLaMA2-Chat-7B} & 1 & 64.4 & -- & -- & 34.6 & -- & -- & -- \\
             \grag: \textcolor{teal}{(i)} + \textcolor{violet}{(ii)} &  & 1 & 86.8 & -- & -- & 62.9 & -- & -- & --\\
             
             \grag: \textcolor{teal}{(i)} + \textcolor{brown}{(iii)} & & 4 & 88.5 & -- & -- & 62.1 & -- & -- & --\\

			\bottomrule
		\end{tabular}%
        
		\end{threeparttable}
}

\end{table*}%

Table~\ref{tab:all_retrieval} has the extended results of Table~\ref{tab:retrieval}, showing performance  results on all three metrics (Hit / H@1 / F1) with respect to the retrieval method used. Overall, \grag improves the vanilla LLM by 149--182\%, when employing the same number of LLM calls for retrieval.

\subsection{Prompt Ablation}

When using RAG, LLM performance depends on the prompts used. To ablate on the prompt impact, we experiment with the following prompts:
\begin{itemize}
 \item \textbf{Prompt A}: 
    \begin{center}
    \fbox{\parbox{4.in}{\raggedright
    \texttt{Based on the \underline{reasoning paths}, please answer the given question. Please keep the answer as simple as possible and return all the possible answers as a list.\textbackslash n \\
    \underline{Reasoning Paths}: \{Reasoning Paths\} \textbackslash n \\
    Question: \{Question\}}}}
    \end{center}
 \item \textbf{Prompt B}: 
    \begin{center}
    \fbox{\parbox{4.in}{\raggedright
    \texttt{Based on the \underline{provided knowledge}, please answer the given question. Please keep the answer as simple as possible and return all the possible answers as a list.\textbackslash n \\
    \underline{Knowledge}: \{Reasoning Paths\} \textbackslash n \\
    Question: \{Question\}}}}
    \end{center}

\item \textbf{Prompt C}: 
    \begin{center}
    \fbox{\parbox{4.in}{\raggedright
    \texttt{\underline{Your tasks is to use the following facts} \underline{and answer the question}. 
    \underline{Make sure that you use the information} \underline{from the facts provided.} Please keep the answer as simple as possible and return all the possible answers as a list.\textbackslash n \\
    \underline{The facts are the following}: \{Reasoning Paths\} \textbackslash n \\
    Question: \{Question\}}}}
    \end{center}
\end{itemize}

\begin{table*}[h]
	\centering
	\caption{Performance comparison (\%Hit) based on different input prompts.}
	\label{tab:prompts}%
	\resizebox{0.5\columnwidth}{!}{
	\begin{threeparttable}
		\begin{tabular}{@{}ll|cc@{}}
			\toprule
			
		      & & \textbf{WebQSP} & \textbf{CWQ} \\
            \midrule
            \multirow{2}{*}{Prompt A} & RoG & 84.8 & 56.4\\
             & \grag & 86.8 & 62.9 \\
             \midrule
             \multirow{2}{*}{Prompt B} & RoG & 84.3 & 55.2\\
             & \grag & 85.2 & 61.7 \\

              \midrule
             \multirow{2}{*}{Prompt C} & RoG & 81.6   &  51.8 \\
             & \grag &  84.4  & 59.4 \\
			
			\bottomrule
		\end{tabular}%
        
		\end{threeparttable}
}

\end{table*}%

We provide the results based on different input prompts in Table~\ref{tab:prompts}. As the results indicate, \grag outperforms RoG in all cases, being robust at the prompt selection.

\subsection{Effect of Training Data}

\captionsetup{width=\columnwidth}
\begin{table}[h]
	\centering
	\caption{Performance results based on different training data.}
	\label{tab:training}%
	\resizebox{\columnwidth}{!}{
		\begin{tabular}{@{}l|ccc|ccc@{}}
			\toprule
            \multirow{2}{*}{Method} & \multicolumn{3}{c|}{\textbf{WebQSP}} & \multicolumn{3}{c}{\textbf{CWQ}} \\ \cmidrule(l){2-4 } \cmidrule(l){5-7 } 
			 & Training Data (Retriever) & Training Data (KGQA Model) & Hit & Training Data (Retriever) & Training Data (KGQA Model) & Hit \\
            \midrule
            UniKGQA & WebQSP & WebQSP & 77.2 & CWQ & CWQ & 51.2 \\
            \midrule
            \multirow{3}{*}{RoG} & WebQSP & WebQSP & 81.5 & CWQ & CWQ & 59.1 \\
            & WebQSP+CWQ & None & 84.8 & WebQSP+CWQ & None & 56.4 \\
             & WebQSP+CWQ & WebQSP+CWQ & 85.7 & WebQSP+CWQ & WebQSP+CWQ & 62.6 \\

			\midrule
            \multirow{2}{*}{\grag} & WebQSP & None & \underline{86.8} & CWQ & None & \underline{62.9} \\
            & WebQSP & WebQSP+CWQ & \textbf{87.2} & CWQ & WebQSP+CWQ & \textbf{66.8} \\
            
			\bottomrule
		\end{tabular}%
		}
\end{table}%

Table~\ref{tab:training} compares performance of different methods based on the training data used for training the retriever and the KGQA model. For example, \grag trains a GNN model for retrieval and uses a LLM for KGQA, which can be fine-tuned or not. As the results show, \grag outperforms the  competing methods (RoG and UniKGQA) by either fine-tuning the KGQA model or not, while it uses the same or less data for training its retriever.

\subsection{Graph Effect} \label{app:subgraph}
GNNs operate on dense subgraphs, which might include noisy information. A question that arises is whether removing irrelevant information from the subgraph would improve GNN retrieval. We experiment with SR~\citep{zhang2022subgraph}, which learns to prune question-irrelevant facts from the KG. As shown in Table~\ref{tab:subgraph}, although SR can improve the GNN reasoning results -- see row (a) vs. (b) at CWQ --, the retrieval effectiveness deteriorates; rows (c) and (d). After examination, we found that the sparse subgraph may contain disconnected KG parts. In this case, \grag's extraction of the shortest paths fails, and \grag returns empty KG information. 

\begin{table*}[h]
	\centering
	\caption{Performance comparison on different subgraphs.}
	\label{tab:subgraph}%
	\resizebox{\columnwidth}{!}{
	\begin{threeparttable}
		\begin{tabular}{@{}ll|cccccc@{}}
			\toprule
			 \multirow{2}{*}{\textbf{\textcolor{olive}{Retriever}}} & \multirow{2}{*}{\textbf{\textcolor{blue}{KGQA Model}}} &  \multicolumn{3}{c}{\textbf{WebQSP}} & \multicolumn{3}{c}{\textbf{CWQ}} \\  \cmidrule(l){3-5} \cmidrule(l){6-8} 
		      & & Hit$^\ast$ & H@1 & F1 & Hit$^\ast$ & H@1 & F1 \\
            \midrule
            a) Dense Subgraph & \textcolor{teal}{(A)} ReaRev + \lmsr  & -- & 77.5 & 72.8 & -- & 52.7 & 49.1  \\
             b) Sparse Subgraph~\citep{zhang2022subgraph}& \textcolor{violet}{(B)} ReaRev + \lmsr &  -- & 74.2 & 69.8 & -- & 53.3 & 49.7   \\
            \midrule
             c) \grag: \textcolor{teal}{(A)} & \multirow{2}{*}{LLaMA2-Chat-7B (tuned)}  &85.0 & 80.3 & 71.5 & 66.2 & 61.3 & 58.9 \\
             d) \grag: \textcolor{violet}{(B)} &  &83.4 & 78.9 & 69.8 & 60.6 & 55.6 & 53.3  \\
			\bottomrule
		\end{tabular}%
        
		\end{threeparttable}
}

\end{table*}%

\section{Limitations} \label{app:limit}

\grag assumes that the KG subgraph, on which the GNN reasons, contains answer nodes. However, as the subgraph is decided by tools such as entity linking and neighborhood extraction, errors in these tools, e.g., unlinked entities, can result to subgraphs that do not include any answers. For example, Table~\ref{tab:data} shows that CWQ subgraphs contain an answer in 79.3\% of the questions. In such cases, \grag cannot retrieve the correct answers to help the LLM answer the questions faithfully. In addition, if the KG has disconnected parts, the shortest path extraction algorithm of \grag may return empty reasoning paths (see Appendix~\ref{app:subgraph}).

\section{Broader Impacts} \label{app:impacts}
\grag is a method that grounds the LLM generations for QA using ground-truth facts from the KG. As a result, \grag can have positive societal impacts by using KG information to alleviate LLM hallucinations in tasks such as QA.

\end{document}